\documentclass[letterpaper,envcountsame]{llncs}  

\usepackage{amsmath,amssymb}
\usepackage{mathtools}
\usepackage{xspace}
\usepackage{macros}
\usepackage{url}
\usepackage{algorithm}
\usepackage[noend]{algorithmic}
\usepackage{nicefrac}
\usepackage{cite}

\usepackage{thmtools}
\usepackage{thm-restate}

\newcommand{\plalc}{\ensuremath{\mathcal{ALCP}}\xspace}
\newcommand{\gci}[1]{\ensuremath{\left<#1\right>}\xspace}
\newcommand{\Texa}{\ensuremath{{\Tmc_{\mathsf{exa}}}}\xspace}

\newcommand{\Kexa}{\ensuremath{{\Kmc_{\mathsf{exa}}}}\xspace}

\newcommand{\logic}{\Lmc}

\newcommand{\probcons}{\Rmc}

\newcommand{\belief}{\Bmc} 
\newcommand{\credbelief}{\Bmc^\csf}
\newcommand{\scepbelief}{\Bmc^\ssf} 
\newcommand{\logicInts}{\ensuremath{\textit{Int}(\logic)}\xspace} 
\newcommand{\logicIntsi}{\ensuremath{\textit{Int}(\logic_i)}\xspace}
\newcommand{\logicIntsp}[1]{\ensuremath{\textit{Int}(\logic_i)}\xspace}

\newcommand{\logicInt}{\ensuremath{\textit{v}}\xspace}
\newcommand{\altlogicInt}{\ensuremath{\textit{w}}\xspace}
\newcommand{\meP}{\mePlong}
\newcommand{\mePlong}{\ensuremath{P^\textit{ME}_{\probcons}}\xspace}
\newcommand{\mePi}{\ensuremath{P^\textit{ME}_{i}}\xspace}
\newcommand{\mePp}[1]{\ensuremath{P^\textit{ME}_{#1}}\xspace}
\newcommand{\Pinduced}{\ensuremath{P^{\Pmc}}\xspace}

\newcommand{\Pinducedp}[1]{\ensuremath{P^{#1}}\xspace}
\newcommand{\propSub}{\ensuremath{\textit{Pr}}\xspace}
\newcommand{\Formula}{\ensuremath{\phi}\xspace}
\newcommand{\altFormula}{\ensuremath{\psi}\xspace}
\newcommand{\meModels}{\ensuremath{\Models_\textit{ME}}\xspace}
\newcommand{\Models}{\ensuremath{\textit{Mod}}\xspace}
\newcommand{\sig}{\ensuremath{\mathsf{sig}}\xspace}
\newcommand{\blaschkeMetric}[1]{\ensuremath{\| #1 \|_B}} 
\newcommand{\varMetric}[1]{\ensuremath{\| #1 \|_1}} 
\newcommand{\eucMetric}[1]{\ensuremath{\| #1 \|_2}} 

\newcommand{\prop}[1]{\scalebox{0.9}{\ensuremath{\textsc{#1}}\xspace}}
\newcommand{\conc}[1]{\scalebox{0.9}{\ensuremath{\mathsf{#1}}\xspace}}
\newcommand{\NCp}[1]{\ensuremath{{\sf N_C^{#1}}}\xspace}
\newcommand{\NRp}[1]{\ensuremath{{\sf N_R^{#1}}}\xspace}

\begin{document}

\title{Probabilistic Reasoning in the Description Logic \plalc with the
		Principle of Maximum Entropy }
\author{
Rafael Pe\~naloza\inst{1} 
\and  
Nico Potyka\inst{2}
}

\institute{Free University of Bozen-Bolzano, Italy \\
			\email{rafael.penaloza@unibz.it}
            \and 
            University of Osnabr\"uck, Germany \\						
            \email{npotyka@uni-osnabrueck.de} }

\maketitle

\begin{abstract}
A central question for knowledge representation is how to encode and
handle uncertain knowledge adequately.
We introduce the probabilistic description logic \plalc that is 
designed for representing context-dependent knowledge,
where the actual context taking place is uncertain. 
\plalc allows the expression of logical dependencies on the domain and 
probabilistic dependencies on the possible contexts.
In order to draw probabilistic conclusions, we employ the 
principle of maximum entropy.
We provide reasoning algorithms for this logic,
and show that it satisfies several 
desirable properties of 
probabilistic logics.
\end{abstract}

\section{Introduction}

A fundamental element of any intelligent application is storing and manipulating the knowledge from the 
application domain. 
Logic-based knowledge representation
languages such as description logics (DLs)~\cite{BCM+-07} provide a
clear syntax and unambiguous semantics that guarantee the correctness of
the results obtained. However, languages based on classical logic are
ill-suited for handling the uncertainty inherent to many application
domains.
To overcome this limitation,
various probabilistic logics have been investigated during the
last three decades (e.g., \cite{Nilsson1986AI,Kern-IsbernerL04,BarnettP08}). 
In particular, several probabilistic DLs
have been developed~\cite{LutzS10,LuSt-JWS08}.
To handle probabilistic knowledge, 
many approaches require
a complete definition of joint probability distributions (JPD)
\cite{RBLZ-12,KlinovP11,Domingos2009,CePe-IJCAR14,dAFL-08}.
One approach to avoid a full JPD specification was proposed by Paris~\cite{Paris94}: 
the user gives a partial specification 
through a set of probabilistic constraints
and the partial knowledge is completed
by means of the principle of maximum entropy.

In this paper we consider a new probabilistic extension of description
logics based on the principle of maximum entropy. 
In our approach
we group different axioms from a knowledge base together into so-called contexts,
which are identified by a propositional formula. Intuitively, each context corresponds
to a possible situation, in which the associated sub-KB is guaranteed
to hold. Uncertainty is associated to the contexts through a set
of probabilistic constraints, which are interpreted under the principle
of maximum entropy. 

To facilitate the understanding of our approach, 
we focus on the DL
\ALC~\cite{SS91} as a prototypical example of a
knowledge representation language, and propositional probabilistic constraints 
as the framework for expressing uncertainty.
As reasoning service we consider subsumption
relations between concepts given some partial
knowledge of the current context.
Since the knowledge in a knowledge base
is typically incomplete, one cannot expect to 
obtain a precise probability for a given consequence. Instead, we compute a belief interval
that describes all the probability degrees that 
can be associated to the consequence without
contradiction.
The lowest bound of the interval corresponds to
a sceptical view, considering only the most
fundamental
models of the knowledge base. The upper
bound, in contrast, reflects the credulous belief
in which every context that is not explicitly 
removed is considered.
In the worst-case, we get the trivial interval 
$[0,1]$, in the best case, we get a point probability  where the upper and lower bounds coincide.  
In some applications, it might be reasonable to consider only one of these bounds. 
For instance, if the probability interval that a treatment will cause heavy complications is $[0.01,0.05]$, we might want to use the upper bound $0.05$. In contrast, when the probability interval that a treatment will be successful is $[0.7,0.9]$, we might be more interested in the lower bound $0.7$.

The main contributions of this paper are the following:
\begin{itemize}
\item we define the new probabilistic description logic \plalc that allows for a flexible description of 
axiomatic dependencies, and its reasoning problems (Section 3);
\item we explain in detail how degrees of belief for the subsumption problem can be computed (Section 4); and
\item we show that \plalc satisfies several desirable properties of probabilistic
logics (Section 5).
\end{itemize}

\section{Maximum Entropy}
We start by recalling the basic notions of probabilistic propositional
logic and the principle of maximum entropy.

Let \logic be a propositional language constructed over a finite signature $\sig(\logic)$, i.e., a set of propositional 
variables, in the usual way. An \logic-interpretation \logicInt is a truth assignment of the propositional variables in
$\sig(\logic)$. \logicInts
denotes the set of all \logic-interpretations.
Satisfaction of a formula $\Formula \in \logic$ by an \mbox{\logic-}interpretation $\logicInt\in\logicInts$ 
(denoted $\logicInt \models \Formula$) is defined as usual.
A probability distribution over $\logic$ is a function 
$P: \logicInts \rightarrow [0,1]$ where 
$\sum_{\logicInt \in \logicInts} P(\logicInt) = 1$. Probability 
distributions are extended to arbitrary \mbox{\logic-}formulas $\Formula$ by setting 
$P(\Formula) = \sum_{\logicInt \models \Formula} P(\logicInt)$.
\begin{definition}[probabilistic constraints, models] Given the propositional language \logic,
a \emph{probabilistic constraint (over \logic)} is an expression 
of the form
\begin{equation}
c_0 + \sum_{i=1}^k c_i \cdot \psf(\Formula_i) \geq 0
\label{eq:constraint}
\end{equation}
where $c_0, c_i \in \mathbb{R}$, and $\Formula_i \in \logic$, $1 \leq i \leq k$.
A probability distribution $P$ over \logic is a \emph{model} of the probabilistic constraint
$c_0+\sum_{i=1}^k c_i \cdot \psf(\Formula_i) \geq 0$ 
if and only if $c_0+\sum_{i=1}^k c_i \cdot P(\Formula_i) \geq 0$.
The distribution $P$ \emph{is a model} of the set of probabilistic contraints \probcons 
($P \models \probcons$) iff it satisfies all the constraints in \probcons.
The set of all models of $\probcons$ is denoted by $\Models(\probcons)$.
If $\Models(\probcons) \neq \emptyset$, we say that 
\probcons is \emph{consistent}.
\end{definition}
Our probabilistic constraints can express the most common types of constraints
considered in the literature of probabilistic logics.
For instance, probabilistic conditionals $(\altFormula \mid \Formula)[\ell,u]$ are satisfied iff 
$\ell \cdot P(\Formula) \leq P(\altFormula \wedge \Formula) \leq u \cdot P(\Formula)$
\cite{Lukasiewicz99probabilisticdeduction}.
That is, the conditional is satisfied iff the conditional probability
of $\altFormula$ given $\Formula$ is between $\ell$ and $u$ whenever $P(\Formula) > 0$. 
Sometimes $P(\Formula) > 0$ is demanded, 
but strict inequalities are computationally
difficult and the semantical differences are
negligible in many  cases, see \cite{Potyka16} for a thorough discussion.
These conditions can be expressed in the form~\eqref{eq:constraint} as follows
\begin{align*}
\psf(\altFormula \land \Formula) - \ell \cdot \psf(\Formula) & {} \geq 0, \qquad \textit{and}\\
u \cdot \psf(\Formula) - \psf(\altFormula \land \Formula) & {} \geq 0.
\end{align*}
Probabilistic constraints can also express more complex restrictions;
for example, we can state that
the probability that a bird cannot fly is at most
one fourth of the probability that a bird flies 
through the constraint
\begin{equation}
\frac{1}{4} \psf(\prop{bird} \land \prop{flies})
- \psf(\prop{bird} \wedge \neg \prop{flies}) 
\geq 0.
\label{exa:flies}
\end{equation}
To improve readability, we will often rewrite constraints in a 
more compact manner, using conditionals as in the first example, or
e.g.\ rewriting~\eqref{exa:flies} as
$
\frac{1}{4} \psf(\prop{bird} \land \prop{flies})
\geq \psf(\prop{bird} \wedge \neg \prop{flies}). 
$

In general, consistent sets of probabilistic constraints have infinitely
many models, and there is no obvious way to distinguish between them. 
One well-studied approach for dealing with this diversity is to focus 
on the model that maximizes the entropy
\begin{equation*}
H(P) = - \sum_{\logicInt \in \logicInts} P(\logicInt) \cdot \log P(\logicInt).
\end{equation*}
From an information-theoretic point of view, the maximum
entropy (ME) distribution can be regarded as the most conservative one in the sense that
it minimizes the information-theoretic distance (that is, the KL-divergence) to the uniform distribution among all probability distributions that satisfy our constraints.
In particular, if there are no restrictions on the probability distributions considered, then the uniform distribution is the ME distribution, 
see, e.g., \cite{yeung2008information} for a more detailed 
discussion of these issues.
A complete characterization of maximum entropy for the purpose of uncertain reasoning can be found in \cite{Paris94}.

\begin{definition}[ME-model]
Let \probcons be a consistent set of probabilistic constraints.
The ME-model \mePlong of $\probcons$ is the \emph{unique} solution of the 
maximization problem
$ 
\arg \max_{P \models \probcons} H(P).
$ 
\end{definition}
Existence and uniqueness of \meP follows from the fact that $H$ is strictly concave and 
continuous, and that the probability distributions that satisfy $\probcons$ form a compact and convex set. \mePlong is usually computed
by deriving an unconstrained optimization problem
by means of the Karush-Kuhn-Tucker conditions.
The resulting problem can be solved, for instance,
by (quasi-)Newton methods with cost $|\logicInts|^3$, see, e.g., \cite{Nocedal2006} for more details on these techniques.


\section{The Probabilistic Description Logic \plalc}

\plalc is a probabilistic extension of the classical description logic \ALC capable
of expressing complex logical and probabilistic relations.
As with classical DLs, the main building blocks in \plalc are \emph{concepts}. 
Syntactically, \plalc concepts are constructed exactly as \ALC concepts.
Given two disjoint sets \NC of \emph{concept names} and \NR of \emph{role names}, 
\plalc concepts are built using the grammar rule
$
C ::= A \mid \neg C \mid C\sqcap C\mid \exists r.C,
$
where $A\in\NC$ and $r\in\NR$.
Note that we can derive disjunction, universal quantification
and subsumption from these rules by using logical equivalences
like $C_1 \sqcup C_2 \equiv \neg (\neg C_1 \sqcap \neg C_2)$.
The knowledge of the application domain is expressed
through a finite set of axioms that restrict the way the
different concepts and roles may be interpreted.
To express both probabilistic and logical relationships,
each axiom is annotated with a formula from \logic that intuitively expresses the
context in which this axiom holds.
\begin{definition}[KB]
An \emph{\logic-restricted general concept inclusion} (\logic-GCI) is of the form 
$\gci{C\sqsubseteq D:\kappa}$ where $C,D$ are \plalc concepts and
$\kappa$ is an \logic-formula.
An \emph{\logic-TBox} is a finite set of \logic-GCIs.
An \plalc \emph{knowledge base}~(KB) over \logic is a pair $\Kmc=(\probcons,\Tmc)$ where \probcons is a set of probabilistic constraints
and \Tmc is an \logic-TBox.
\end{definition}
\begin{example}
\label{exa:antibiotics_degrees_of_belief}
Consider an application modeling beliefs about bacterial 
and viral infections using the concept names
\conc{strep} (streptococcal infection),
\conc{bac} (bacterial infection),
\conc{vir} (viral infection),
\conc{inf} (infection), and 
\conc{ab} (antibiotic);
and the role names
\conc{sf} (suffers from), and
\conc{suc} (successful treatment);
and the propositional variables
\prop{res} (antibiotic resistance), and
\prop{h} (heavy use of antibiotics by patient).
Define the \logic-TBox \Texa containing the \logic-GCIs
\begin{align*}
&{\gci{\exists\conc{sf}.\conc{bac} \sqsubseteq \exists\conc{suc}.\conc{ab}: \neg \prop{res}{\land} \neg\prop{h}}},
&\gci{\exists\conc{sf}.\conc{vir} \sqsubseteq \neg \exists\conc{suc}.\conc{ab}: \top},
&\quad \gci{\conc{strep} \sqsubseteq \conc{bac}: \top},\\
&{\gci{\exists\conc{sf}.\conc{bac} \sqsubseteq \neg \exists\conc{suc}.\conc{ab}: \prop{res}}}, 
&{\gci{\conc{bac} \sqsubseteq \conc{inf}: \top}},
&\quad \gci{\conc{vir} \sqsubseteq \conc{inf}: \top},
\end{align*}
where $\top$ is any \logic-tautology. 
For example, the first \logic-GCI states that a bacterial infection can be treated successfully with antibiotics if no antibiotic resistance is present and there was no heavy use of antibiotics; 
the second one states that viral infections can never be treated with antibiotics successfully. 
Consider additionally the set \probcons containing the probabilistic constraints
containing
\begin{align*}
&(\prop{res})[0.05], 
&(\prop{res} \mid \prop{h})[0.8].
\end{align*}
That is, the probability of an antibiotic resistance is $5\%$ if no further information
is given. However, if the patient used antibiotics heavily, the probability increases to $80\%$.
\end{example}
%
Notice that the probabilistic constraints, and hence the representation
of the uncertainty in the knowledge, refer only to the propositional
formulas that label the \logic-GCIs. In \plalc, the uncertainty of the knowledge
is handled through these propositional formulas as explained next.

A possible world interprets both the axiom language (i.e., the concept
and role names) and the context language (the propositional variables).
Intuitively, it 
describes a possible context (\mbox{\logic-}interpretation) together with the relationships
between concepts in that situation (\alc-interpretation).
\begin{definition}[possible world]
A \emph{possible world} is a triple 
$\Imc=(\Delta^\Imc,\cdot^\Imc,\logicInt^\Imc)$ where $\Delta^\Imc$ is a non-empty set 
(called the \emph{domain}),
$\logicInt^\Imc$ is an \logic-interpretation, and $\cdot^\Imc$ is an 
\emph{interpretation function} that maps every concept name $A$ to a set 
$A^\Imc\subseteq\Delta^\Imc$ and 
every role name $r$ to a binary relation $r^\Imc\subseteq\Delta^\Imc\times\Delta^\Imc$.
\end{definition}
The interpretation function $\cdot^\Imc$ is extended to complex concepts 
as usual in DLs by letting
$(\neg C)^\Imc:=\Delta^\Imc\setminus C^\Imc$; 
$(\exists r.C)^\Imc:=\{d\in\Delta^\Imc\mid \exists e\in\Delta^\Imc.(d,e)\in r^\Imc, e\in C^\Imc\}$; and
$(C\sqcap D)^\Imc:=C^\Imc\cap D^\Imc$.
A possible world
is a model of an \logic-GCI iff it
satisfies the description logic constraint of the axiom
whenever it satisfies the context.

\begin{definition}[model of TBox]
\label{def:model}
The possible world $\Imc=(\Delta^\Imc,\cdot^\Imc,\logicInt^\Imc)$ is a \emph{model} of the \mbox{\logic-}GCI \gci{C\sqsubseteq D:\kappa}, denoted as 
$\Imc\models\gci{C\sqsubseteq D:\kappa}$, iff (i) $\logicInt^\Imc\not\models\kappa$, or (ii) $C^\Imc\subseteq D^\Imc$.
It is a \emph{model} of the \logic-TBox \Tmc iff it is a model of all the \mbox{\logic-}GCIs in \Tmc.
\end{definition}
The classical DL \ALC is a special case of \plalc where all the axioms are annotated
with an \logic-tautology $\top$. To preserve the syntax of classical DLs, we denote
such \logic-GCIs as $C\sqsubseteq D$ instead of $\gci{C\sqsubseteq D:\top}$.
In this case, the condition~(i) from Definition~\ref{def:model} 
cannot be satisfied, and hence a model is required to
satisfy $C^\Imc\subseteq D^\Imc$ for all \logic-GCIs $C\sqsubseteq D$ in the TBox.
For a deeper introduction to classical \ALC, see~\cite{BCM+-07}.

According to our semantics, we only demand that the \logic-GCIs are satisfied in some
specific contexts. Thus, it is often useful to focus on the classical \ALC TBox
that contains the knowledge that holds in a particular situation.
%
For a KB $\Kmc=(\probcons,\Tmc)$ and $\logicInt\in\logicInts$, the 
\emph{\logicInt-restricted TBox} is 
the \ALC TBox
\[
\Tmc_\logicInt := \{ C\sqsubseteq D \mid \gci{C\sqsubseteq D:\kappa}\in\Tmc, \logicInt\models\kappa\}. 
\]
%
The possible world \Imc satisfies $\Tmc_\logicInt$ ($\Imc\models\Tmc_\logicInt$) 
if for all \logic-GCIs $C\sqsubseteq D\in\Tmc_\logicInt$ it holds that $C^\Imc\subseteq D^\Imc$.
In the following, we will often consider \emph{subsumption} and \emph{strong non-subsumption}
between concepts w.r.t.\ a restricted TBox. We say that $C$ is \emph{subsumed} by
$D$ w.r.t.\ $\Tmc_\logicInt$ ($\Tmc_\logicInt\models C\sqsubseteq D$) if 
for every $\Imc\models\Tmc_\logicInt$ it holds that $C^\Imc\subseteq D^\Imc$. Dually,
$C$ is \emph{strongly non-subsumed} by
$D$ w.r.t.\ $\Tmc_\logicInt$ ($\Tmc_\logicInt\models C\not\,\not\sqsubseteq D$) if 
for every $\Imc\models\Tmc_\logicInt$, $C^\Imc\not\subseteq D^\Imc$ holds.
Notice that strong non-subsumption requires that the inclusion between axioms does not hold
in \emph{any} possible world satisfying $\Tmc_\logicInt$. Hence, this condition is more strict
than just negating the subsumption relation.

We now describe how the probabilistic constraints are handled in our logic.
An \plalc-interpretation consists of a finite set of possible worlds and a probability
function over these worlds.
\begin{definition}[\plalc-interpretation]
\label{def:prob:interpretation}
An \emph{\mbox{\plalc-}in\-ter\-pre\-ta\-tion} is a pair of the form 
{$\Pmc=(\Imf,P_\Imf)$}, 
where \Imf is a non-empty, finite set of possible worlds
and $P_\Imf$ is a probability distribution over \Imf. 
\end{definition}
Each \plalc-interpretation induces a probability distribution over \logic.
The probability of a context can be obtained by adding the probabilities 
of all possible worlds in which this context holds. 
\begin{definition}[distribution induced by $\Pmc$]
Let $\Pmc=(\Imf,P_\Imf)$ be an \mbox{\plalc-}in\-ter\-pre\-ta\-tion.
The probability distribution $\Pinduced: \logicInts \rightarrow [0,1]$ \emph{induced} 
by \Pmc is defined by
$\Pinduced(\logicInt) := \sum_{\Imc \in \Imf|_{\logicInt}} P_\Imf(\Imc)$,
where $\Imf|_{\logicInt} = \{ (\Delta^\Imc,\cdot^\Imc,\logicInt^\Imc) \in \Imf 
\mid \logicInt^\Imc = \logicInt \}$.
\end{definition}
As usual, reasoning is restricted to interpretations that satisfy the restrictions imposed
by the knowledge base.
In our case, we have to demand that the interpretation is consistent
with both the classical and the probabilistic part of our knowledge base.
That is, we consider only those possible worlds that satisfy both the terminological
knowledge (\Tmc) and the probabilistic constraints (\probcons).
\begin{definition}[model]
\label{def:prob:model}
Let $\Pmc=(\Imf,P_\Imf)$ be an
\mbox{\plalc-}in\-ter\-pre\-tation.
\Pmc is consistent with the TBox \Tmc if 
every $\Imc\in\Imf$ is a model of \Tmc.
\Pmc is consistent with the set of probabilistic constraints \probcons iff $\Pinduced \models \probcons$.
The \plalc-interpretation \Pmc is a \emph{model} of the KB 
$\Kmc=(\probcons,\Tmc)$ iff it is consistent with both \Tmc and \probcons.
As usual, a KB is \emph{consistent} iff it has a model.
\end{definition}
Notice that \plalc-KBs can express both, logical and probabilistic 
dependencies between axioms. For instance, two \logic-GCIs
$\gci{C_1\sqsubseteq D_1:\kappa_1}$ and $\gci{C_2\sqsubseteq D_2:\kappa_2}$
where $\kappa_1\Rightarrow\kappa_2$ express that whenever the first
\logic-GCI is satisfied, the second one must also hold. Similarly, the probabilistic
dependencies between axioms are expressed via the probabilistic constraints
of the labeling formulas.

We are interested in computing degrees of belief for subsumption relations
between concepts.
We define the conditional probability of a subsumption relation given a context 
with respect to a given \plalc-interpretation following the usual notions of 
conditioning.
\begin{definition}[probability of subsumption]
Let $C,D$ be concepts, $\kappa$ a context and $\Pmc$ an \plalc-interpretation. 
The \emph{conditional probability of $C\sqsubseteq D$ given $\kappa$} with respect to 
$\Pmc$ is
\begin{align}
\label{eq:DefCondProbSubsumpt}
\propSub_\Pmc(C\sqsubseteq D \mid \kappa):=
\frac{\sum_{\Imc\in\Imf,\Imc \models \kappa,\Imc\models C\sqsubseteq D}P_\Imf(\Imc)}
{\sum_{\Imc\in\Imf,\Imc \models \kappa}P_\Imf(\Imc)}.
\end{align}
\end{definition}
%
Notice that the denominator in \eqref{eq:DefCondProbSubsumpt} can be rewritten as
\begin{align*}
\sum_{\Imc\in\Imf,\Imc \models \kappa}P_\Imf(\Imc)
&= \sum_{\logicInt \models \kappa} \sum_{\Imc\in\Imf|_\logicInt}P_\Imf(\Imc)
= \sum_{\logicInt \models \kappa} \Pinduced(\logicInt) 
= \Pinduced(\kappa).
\end{align*}
As usual, the conditional probability is only well-defined when $\Pinduced(\kappa)>0$.

Recall that  the set of probabilistic constraints \probcons may be satisfied
by an infinite class of probability distributions. 
In the spirit of maximum entropy reasoning,
we consider only the most conservative ones
in the sense that they induce the ME-model \meP of \probcons.
\begin{definition}[ME-\plalc-model]
\label{def:prob:memodel}
An \plalc-model \Pmc of \Kmc
 is called an \emph{ME\mbox{-}\plalc-model of \Kmc} iff
$\Pinduced = \meP$.
The set of all ME-\plalc-models of $\Kmc$ is denoted by
$\meModels(\Kmc)$.
\Kmc is called \emph{ME-consistent} iff $\meModels(\Kmc) \neq \emptyset$. 
\end{definition}
Note that ME-consistency is a strictly stronger notion of consistency.
ME-con\-sis\-tent knowledge bases are always consistent, but 
the converse does not necessarily hold if the classical
TBox obtained from $\Tmc$ by restricting to a context is
inconsistent as we show in the following example.
\begin{example}
Let $\sig(\logic) = \{x\}$ and 
$\Kmc=(\probcons,\Tmc)$ be the KB with $\probcons = \emptyset$ and 
$\Tmc = \{\gci{A\sqcup\neg A \sqsubseteq A\sqcap\neg A: x}\}$.
Since $A\sqcup\neg A \sqsubseteq A\sqcap\neg A$ is contradictorial,
each \plalc-model of \Kmc must satisfy  $\neg x$.
There certainly are such models, but in each such model \Pmc,
$\Pinduced(x) = 0$.
However, since $\probcons = \emptyset$, we have $\meP(x) = 0.5$ 
and hence $\Kmc$ has no ME-model.
\end{example}
ME-inconsistency rules out
some undesired cases in which the whole knowledge base is 
consistent, but the TBox restricted to some context is 
inconsistent.
The following theorem gives a simple characterization of 
ME-consistency: to verify ME-consistency of a KB, it suffices to 
check consistency of the TBoxes induced by the 
\mbox{\logic-}interpretations that have positive probability 
with respect to $\meP$. By the properties of the ME distribution, these are the interpretations
that are not explicitly restricted to have zero probability through \probcons.
%
\begin{restatable}{theorem}{consistency}
The KB $\Kmc=(\probcons,\Tmc)$ is ME-consistent iff for every
$\logicInt\in\logicInts$ such that $\meP(\logicInt)>0$, $\Tmc_\logicInt$ is consistent.
\end{restatable}
%
\noindent
For the rest of this paper we consider only ME\mbox{-}consistent
KBs. Hence, whenever we speak of a KB \Kmc, we implicitly assume
that \Kmc has at least one ME-model.

We are interested in computing the probability of a subsumption relation w.r.t.\ 
a given KB \Kmc. Notice that, although we consider only one probability distribution
\meP, there can still exist many different ME-models of \Kmc, which yield different
probabilities for the same subsumption relation. One way to handle this is to consider the 
smallest and largest probabilities that can be consistently associated to this relation.
We call them the \emph{sceptical} and the \emph{creduluos} degrees of belief, 
respectively.
\begin{definition}[degree of belief]
Let $C,D$ be \plalc concepts, 
$\kappa$ a context,
and $\Kmc=(\probcons,\Tmc)$ an \plalc KB. 
The \emph{sceptical degree of belief of $C\sqsubseteq D$ given $\kappa$ w.r.t.\ \Kmc} is 
\[
\scepbelief_\Kmc(C\sqsubseteq D  \mid \kappa):=
			\inf_{\Pmc \in \meModels(\Kmc)}\propSub_\Pmc(C\sqsubseteq D \mid \kappa).
\]
The \emph{credulous degree of belief of $C\sqsubseteq D$ given $\kappa$ w.r.t.\ \Kmc} is 
\[
\credbelief_\Kmc(C\sqsubseteq D  \mid \kappa):=
			\sup_{\Pmc \in \meModels(\Kmc)}\propSub_\Pmc(C\sqsubseteq D  \mid \kappa).
\]
\end{definition}
\begin{example}
\label{exa:antibiotics_degrees_of_belief:2}
Consider $\Kexa$
from Example~ \ref{exa:antibiotics_degrees_of_belief}.
If we ask for the degrees of belief that a patient who suffers
from an infection can be successfully treated with antibiotics,
we obtain
\begin{align*}
\scepbelief_\Kexa(\exists\conc{sf}.\conc{inf} \sqsubseteq \exists\conc{suc}.\conc{ab}\mid \top)
& {} = 0,\\
\credbelief_\Kexa(\exists\conc{sf}.\conc{inf} \sqsubseteq \exists\conc{suc}.\conc{ab}\mid \top)
& {} = 1.
\end{align*}
These bounds are not very informative, but they are perfectly justified by our knowledge
base since we do not know anything about the effectiveness of antibiotics with
respect to infections in general.
However, for a patient who suffers from a streptococcal infection we get
\begin{align*}
\scepbelief_\Kexa(\exists\conc{sf}.\conc{strep} \sqsubseteq \exists\conc{suc}.\conc{ab}\mid \top)
& {} = 0.9405,\\
\credbelief_\Kexa(\exists\conc{sf}.\conc{strep} \sqsubseteq \exists\conc{suc}.\conc{ab}\mid \top)
& {} = 0.95.
\end{align*}
If we know that this patient used antibiotics heavily in the past, then
there is nothing in our knowledge base that guarantees the existence of a
successful treatment. Hence, the degrees of belief become
\begin{align*}
\scepbelief_\Kexa(\exists\conc{sf}.\conc{strep} \sqsubseteq \exists\conc{suc}.\conc{ab}\mid \prop{h})
& {} = 0\\
\credbelief_\Kexa(\exists\conc{sf}.\conc{strep} \sqsubseteq \exists\conc{suc}.\conc{ab}\mid \prop{h})
& {} = 0.2.\\
\end{align*}
\end{example}
Our definition of the sceptical degree of belief raises a
 philosophical question: should there be no
difference between the degree of belief $0$ and an
infinitely small degree of belief? A dual question arises for the credulous degree of 
belief and the probability $1$. However,
as we show in the next section, the sceptical and credulous 
degrees of belief actually correspond to minimum and
maximum rather than to infimum and supremum
(see Corollary~\ref{cor:witness})
so that these questions become vacuous.
From the following theorem we can conclude that every
intermediate degree can also be obtained by some model of 
the KB.
\begin{restatable}[Intermediate Value Theorem]{theorem}{IVT}
Let $p_1<p_2$ and $\Pmc_1$ and $\Pmc_2$ be two ME\mbox{-\plalc-}models of the KB
$\Kmc=(\probcons,\Tmc)$ such that  
$\propSub_{\Pmc_1}(C\sqsubseteq D\mid\kappa) = p_1$ and 
$\propSub_{\Pmc_2}(C\sqsubseteq D\mid\kappa) = p_2$.
Then for each $p$ between $p_1$ and $p_2$ there exists an
ME-\plalc-model $\Pmc$ of \Kmc such that
$\propSub_{\Pmc}(C\sqsubseteq D\mid\kappa) = p$
\end{restatable}
\noindent
As we will show in Corollary~\ref{cor:witness}, both the sceptical degree 
$\scepbelief_\Kmc(C\sqsubseteq D\mid\kappa)$ and 
the credulous degree
$\credbelief_\Kmc(C\sqsubseteq D\mid\kappa)$
are in fact witnessed by some ME\mbox{-}models. Therefore it is 
meaningful to consider the whole interval of 
beliefs between $\scepbelief_\Kmc(C\sqsubseteq D\mid\kappa)$ and 
$\credbelief_\Kmc(C\sqsubseteq D\mid\kappa)$.
\begin{definition}[belief interval]
Let $C,D$ be \plalc concepts, 
$\kappa \in \logic$ a context 
and $\Kmc=(\probcons,\Tmc)$ a \plalc KB. 
The \emph{belief interval} for $C\sqsubseteq D$ w.r.t.\ \Kmc given $\kappa$ is 
\[
\belief_\Kmc(C\sqsubseteq D \mid \kappa):= 
[\scepbelief_\Kmc(C\sqsubseteq D \mid \kappa), \credbelief_\Kmc(C\sqsubseteq D \mid \kappa)].
\]
\end{definition}
%
%
\section{Computing Beliefs}

In this section we show how to compute the belief interval.
The first theorem states that the sceptical degreef of belief for a subsumption relation
can be computed by adding the probabilities of those \logic-interpretations $w$ that entail this 
subsumption in the corresponding restricted TBox $\Tmc_w$.
\begin{restatable}{theorem}{sceptical}
\label{thm:precise}
Let $\Kmc=(\probcons,\Tmc)$ be a KB, $C,D$ two concepts, and $\kappa$ a context
such that $\meP(\kappa)>0$. 
Then
\[
\scepbelief_\Kmc(C\sqsubseteq D\mid \kappa)=
	\frac{\sum_{{w\in\logicInts},{\Tmc_w\models C\sqsubseteq D},{w\models\kappa}}\meP(w)}
    	 {\meP(\kappa)}.
\]
\end{restatable}
\noindent
Dually, the credulous degree of belief for a subsumption relation can be computed by removing all the 
situations in which this relation cannot possibly hold.
%
\begin{restatable}{theorem}{credulous}
\label{thm:precise:credulous}
Let $\Kmc=(\probcons,\Tmc)$ be a KB, $C,D$ two concepts, and $\kappa$ a context with
$\meP(\kappa)>0$.
Then
\[
\credbelief_\Kmc(C\sqsubseteq D\mid\kappa)=
	1 - \frac{\sum_{{w\in\logicInts},{\Tmc_w\models C\not\,\not\sqsubseteq D},{w\models\kappa}}\meP(w)}
    		{\meP(\kappa)}.
\]
\end{restatable}
\noindent
To prove these theorems, one can build two models of the KB \Kmc,
\Pmc and \Qmc such that 
$\propSub_{\Pmc}(C\sqsubseteq D\mid\kappa)$ and 
$\propSub_{\Qmc}(C\sqsubseteq D\mid\kappa)$ are those degrees expressed
by Theorems~\ref{thm:precise} and \ref{thm:precise:credulous}, respectively.
As a byproduct of these proofs, we obtain that the infimum and supremum
that define the sceptical and  the credulous degrees of belief actually correspond to minimum and maximum taken by some ME-models, 
yielding the following corollary.
\begin{corollary}
\label{cor:witness}
Let \Kmc be an \plalc KB, $C,D$ be two concepts, and $\kappa$ be a
context. There exist two
ME-models \Pmc,\Qmc of \Kmc with
$\scepbelief_\Kmc(C\sqsubseteq D\mid \kappa)=\propSub_\Pmc(C\sqsubseteq D\mid \kappa)$
and 
$\credbelief_\Kmc(C\sqsubseteq D\mid \kappa)=\propSub_\Qmc(C\sqsubseteq D\mid \kappa)$.
\end{corollary}
The direct consequence of Theorems \ref{thm:precise} and \ref{thm:precise:credulous} is that 
if we want to compute the belief interval
for $C\sqsubseteq D$ given some context, it suffices to identify all
\mbox{\logic-}interpretations whose induced (classical) TBoxes entail the subsumption
relation $C\sqsubseteq D$ (for the sceptical belief) or the strong non-subsumption
$C\not\,\not\sqsubseteq D$ (for credulous belief). Recall that every set of propositional
interpretations can be represented by a propositional formula. This motivates
the following definition.
\begin{definition}[consequence formula]
\label{def:cf}
An \logic-formula $\Formula$ is a \emph{consequence formula} for 
$C\sqsubseteq D$ (respectively $C\not\,\not\sqsubseteq D$) w.r.t.\ the \logic-TBox \Tmc
if for every $w\in\logicInts$ it holds that 
$w\models\phi$ iff $\Tmc_w\models C\sqsubseteq D$ 
(respectively $\Tmc_w\models C\not\,\not\sqsubseteq D$).
\end{definition}
%
%
If we are able to compute these consequence formulas, then the computation of the 
belief interval can be reduced to the evaluation of the probability of these 
formulas w.r.t.\ the ME\mbox{-}distribution satisfying \probcons. 
%
\begin{restatable}{theorem}{consform}
\label{thm:cf}
Let $\Kmc=(\probcons,\Tmc)$ be an \plalc KB, $\Formula$ and $\altFormula$ be
consequence formulas for $C\sqsubseteq D$ and $C\not\,\not\sqsubseteq D$
w.r.t.\ \Tmc, respectively, and $\kappa$ a context. Then 
$\scepbelief_\Kmc(C\sqsubseteq D\mid\kappa)=\meP(\Formula\mid\kappa)$ and 
$\credbelief_\Kmc(C\sqsubseteq D\mid\kappa)=1-\meP(\altFormula\mid\kappa)$.
\end{restatable}
\begin{example}
In our running example, one can see that
a consequence formula for 
$\exists\conc{sf}.\conc{strep}\sqsubseteq\exists\conc{suc}.\conc{ab}$ is
$\neg\prop{res}\land\neg\prop{h}$. Indeed, in order to deduce
this consequence it is necessary to satisfy the first axiom
of \Texa, which is only guaranteed in the context
$\neg\prop{res}\land\neg\prop{h}$. 
Similarly, $\prop{res}$ is a consequence formula for 
$\exists\conc{sf}.\conc{strep}\not\,\not\sqsubseteq\exists\conc{suc}.\conc{ab}$.
Knowing both the consequence formulas and the ME-model, we can deduce
$$\scepbelief_\Kexa(\exists\conc{sf}.\conc{strep} \sqsubseteq \exists\conc{suc}.\conc{ab}\mid \top)=
	\meP(\neg\prop{res}\land\neg\prop{h})=0.9405
$$
and
$$\credbelief_\Kexa(\exists\conc{sf}.\conc{strep} \sqsubseteq \exists\conc{suc}.\conc{ab}\mid \prop{h})=
	1-\meP(\prop{res}\mid\prop{h})=0.2.
$$
\end{example}
In particular, Theorem~\ref{thm:cf} implies that the belief interval can be computed in two phases. 
The first phase uses purely logical reasoning to compute the consequence formulas,
while the second phase applies probabilistic inferences to compute the degrees of
belief from these formulas. We now briefly explain how the consequence formulas
can be computed.

Notice first that subsumption and non-subsumption are monotonic consequences
in the sense of~\cite{BaKP-JWS12}; that is, 
if an \ALC TBox \Tmc entails the subsumption $C\sqsubseteq D$, then every
superset of \Tmc also entails this consequence. Similarly, adding more axioms to
a TBox entailing $C\not\,\not\sqsubseteq D$ does not remove this entailment.
Moreover, the set of all \mbox{\logic-}formulas (modulo logical equivalence) forms a
distributive lattice ordered by generality, in which \logic-interpretations are
all the join prime elements. Thus, the consequence formulas from 
Definition~\ref{def:cf} are in fact the so-called \emph{boundaries} 
from~\cite{BaKP-JWS12}. Hence, they can be computed using any of the known
boundary computation approaches. 

Assuming that the number of contexts is small in comparison to
the size of the TBox, it is better to compute the degrees
of belief through a more direct approach following 
Theorems~\ref{thm:precise} and~\ref{thm:precise:credulous}. In order to compute
$\scepbelief_\Kmc(C\sqsubseteq D\mid\kappa)$ and
$\credbelief_\Kmc(C\sqsubseteq D\mid\kappa)$, it suffices to enumerate all 
interpretations $\logicInt\in\logicInts$ and check whether 
$\Tmc_\logicInt\models C\sqsubseteq D$ or 
$\Tmc_\logicInt\models C\not\,\not\sqsubseteq D$,
and $\logicInt\models\kappa$, or not (see Algorithm~\ref{alg:scepbel}).
\begin{algorithm}[tb]
\caption{Computing degrees of belief}
\label{alg:scepbel}
\begin{algorithmic}
\REQUIRE KB $\Kmc=(\probcons,\Tmc)$, concepts $C,D$, context $\kappa$
\ENSURE Belief degrees
	$\big(\scepbelief_\Kmc(C\sqsubseteq D{\mid}\kappa),\
	 \credbelief_\Kmc(C\sqsubseteq D{\mid}\kappa)\big)$
\STATE $\ell_s\gets 0$; $\ell_c\gets 0$    
\FORALL{$\logicInt\in\logicInts$} 
	\IF{$\logicInt\models\kappa$}
    	\IF{$\Tmc_\logicInt\models C\sqsubseteq D$}
        	\STATE $\ell_s\gets\ell_s+\meP(\logicInt)$
        \ELSIF{$\Tmc_\logicInt\models C\not\,\not\sqsubseteq D$}
        	\STATE $\ell_c\gets\ell_c+\meP(\logicInt)$        
        \ENDIF
    \ENDIF
\ENDFOR 
\RETURN $\left({\ell_s}/{\meP(\kappa)},1-{\ell_c}/{\meP(\kappa)}\right)$
\end{algorithmic}
\end{algorithm}
This approach requires $2^{|\sig(\logic)|}$ calls to a standard \ALC reasoner, and
each of these calls runs in exponential time on $|\Tmc|$~\cite{DoMa00}. 
Notice that this algorithm has an \emph{any-time} behaviour:
it is possible to stop its execution at any moment and obtain
an approximation of the belief interval. Moreover, the longer
the algorithm runs, the better this approximation becomes. Thus,
this method is adequate for a system where finding good 
approximations efficiently may be more important than computing
the precise answers.





\section{Properties}

We now investigate some properties of probabilistic logics
 \cite{Paris94}.
 First we show that \plalc is \emph{language and representation invariant}.
Invariance is meant with respect to logical objects. Language invariance means that just
extending the language without changing the knowledge base should not affect reasoning results. Representation invariance
means that equivalent knowledge bases should yield equal inference results.
Notice that different notions of \emph{representation dependence} exist in the
literature. For instance, in \cite{halpern2004representation} a very different notion
is considered, where the language and the knowledge base are changed simultaneously. This case is not covered by our notion of representation invariance. 
\plalc also satisfies an \emph{independence} property; i.e., reasoning 
results about a part of the language are not changed,
when we add knowledge about an independent part of the
language. Finally, \plalc is \emph{continuous} in the sense
that minor changes in the probabilistic knowledge expressed 
by a knowledge base cannot induce major changes in the
reasoning results.

\begin{restatable}[Representation invariance]{theorem}{repinv}
 Let $\Kmc_i=(\probcons_i,\Tmc_i)$, $i\in\{1,2\}$, be two
 KBs such that
 $\Models(\probcons_1) = \Models(\probcons_2)$ and
$\Models(\Tmc_1) = \Models(\Tmc_2)$.
Then 
for all concepts $C,D$ and contexts $\kappa \in \logic$,
$\belief_{\Kmc_1}(C\sqsubseteq D \mid \kappa) = \belief_{\Kmc_2}(C\sqsubseteq D \mid \kappa)$.
\end{restatable}
\noindent
\plalc is not only representation invariant, but also language invariant. 
This property is of computational interest, in particular
in combination with independence, that we investigate subsequently.
To illustrate this, suppose that we added knowledge about bone fractures in
our medical example, which is independent of the knowledge about
infections. Independence guarantees that we can ignore the 
knowledge about infections when answering queries about bone
fractures. 
In this way, we can decrease the size of the knowledge base. 
Language invariance guarantees that we can also ignore 
the concepts, relations and propositional variables related to
the infection domain. Thus, we can decrease the size
of the language. Exploiting both properties, the size of 
the computational problems can sometimes be decreased significantly.
\begin{restatable}[Language Invariance]{theorem}{langinv}
Let $\Kmc_1,\Kmc_2$ be KBs over $\logic^1, \NCp{1}, \NRp{1}$ and  $\logic^2, \NCp{2}, \NRp{2}$, respectively.
If $\Kmc_1 = \Kmc_2$, $\logic^1 \subseteq \logic^2, \NCp{1} \subseteq \NCp{2}$ and $\NRp{1} \subseteq \NRp{2}$, then for all concepts $C,D \in \NCp{1}$ and contexts $\kappa \in \logic^1$, it holds that
$$\belief_{\Kmc_1}(C\sqsubseteq D \mid \kappa) = \belief_{\Kmc_2}(C\sqsubseteq D \mid \kappa).$$
\end{restatable}
\noindent
For an \logic-TBox \Tmc, we define the \emph{signature} of \Tmc to be the 
set $\sig(\Tmc)$ of all concept names and role names appearing in \Tmc. Likewise,
$\sig(\probcons)$ is the set of all propositional variables appearing in \probcons.
The signature of a KB $\Kmc=(\probcons,\Tmc)$ is 
$\sig(\Kmc):=\sig(\probcons)\cup\sig(\Tmc)$.
\begin{restatable}[Independence]{theorem}{independence}
\label{thm:independence}
Let $\Kmc_1,\Kmc_2$ be s.t.\ $\sig(\Kmc_1)\cap\sig(\Kmc_2)=\emptyset$, 
$C,D$ be two concepts, and $\kappa$ a context 
where
$\left(\sig(C)\cup\sig(D)\cup\sig(\kappa)\right)\cap \sig(\Kmc_2)=\emptyset$.
Then $\belief(C\sqsubseteq_{\Kmc_1} D\mid\kappa)=\belief(C\sqsubseteq_{\Kmc_1\cup\Kmc_2} D\mid\kappa)$. 
\end{restatable}
\noindent
The last property we consider is continuity. One important practical
feature of continuous probabilistic logics is that they guarantee a numerically stable behaviour. That is, 
minor rounding errors due to floating-point arithmetic
will not result in major errors in the computed probabilities.
As demonstrated by Paris in~\cite{Paris94}, measuring the difference 
between probabilistic knowledge bases is subtle and is best addressed 
by comparing knowledge bases extensionally; i.e.,
with respect to their model sets.
To this end, Paris considered the Blaschke metric. 
 Formally, the \emph{Blaschke distance} $\blaschkeMetric{S_1, S_2}$ between two convex
sets $S_1, S_2$ is defined by
\begin{align*}
\inf \{\delta \in \mathbb{R} \mid \
&\forall P_1 \in S_1 \exists P_2 \in S_2: \eucMetric{P_1, P_2} \leq \delta
\ \textit{and} \\
&\forall P_2 \in S_2 \exists P_1 \in S_1:\eucMetric{P_2, P_1} \leq \delta
\}
\end{align*}
Intuitively, $\blaschkeMetric{S_1, S_2}$ is the smallest real number $d$ such that for each distribution in one of the sets,
there is a probability distribution in the other that has distance at most $d$ to the former.
We say that a sequence of knowledge bases $(\Kmc_i)$
converges to a knowledge base \Kmc
iff the classical part of each $\Kmc_i$ is equivalent to the classical part of $\Kmc$
and the probabilistic part converges to the 
probabilistic part of $\Kmc$.
Our reasoning approach behaves indeed continuously with respect to this metric. 
\begin{restatable}[Continuity]{theorem}{continuity}
Let $(\Kmc_i)$ be a convergent sequence of KBs
with limit $\Kmc$ and $\belief_{\Kmc_i}(C\sqsubseteq D\mid\kappa) = [\ell_i,u_i]$.
If $\belief_{\Kmc}(C\sqsubseteq D\mid\kappa) = [\ell,u]$,
then $(l_i)$ converges to $\ell$ and $(u_i)$ converges to $u$ (with respect to the usual 
topology on $\mathbb{R}$). 
\end{restatable}

\section{Related Work}

Relational probabilistic logical approaches can be roughly divided into those that consider probability distributions over the domain, those that consider probability
distributions over possible worlds and those that combine both ideas \cite{Halpern90}. 
Our framework belongs to the second group.
Maximum entropy reasoning in propositional probabilistic logics has been discussed
extensively, e.g., in~\cite{Paris94,Kern-Isberner00d}, and
various extensions to first-order languages have been 
considered in recent years~\cite{Kern-IsbernerL04,BarnettP08,Kern-IsbernerThimm09a,BeierleKFP15}.
In these works, the domain is restricted to a finite number of constants or bounded in the limit. We circumvent the need to do so by combining a classical first-order logic with unbounded domain with a
probabilistic logic with fixed domain.

Many probabilistic DLs have also been considered in the last
decades~\cite{LuSt-JWS08,LutzS10,KlinovP11}.
Our approach is closest to Bayesian DLs~\cite{CePe-IJCAR14,dAFL-08} and {\sc{disponte}}~\cite{RBLZ-12}. The greatest
difference with the former lies in the fact that
\plalc KBs do not require a complete specification
of the probability distribution, but only a set
of probabilistic constraints. Moreover, the previous
formalisms consider only the sceptical degree of 
belief, while we are interested in the full belief
interval. 
In contrast to {\sc{disponte}}, \plalc is capable of expressing
both, logical and probabilistic dependencies between the axioms
in a KB; in addition, {\sc{disponte}} requires all uncertainty
degrees to be assigned as mutually independent point probabilities, 
while \plalc allows for a more flexible specification.

\section{Conclusions}

We have introduced the probabilistic DL \plalc, which extends the classical DL \alc 
with the capability of expressing and reasoning about uncertain contextual knowledge
defined through the principle of maximum entropy.
Effective reasoning methods were developed using the decoupling between the
logical and the probabilistic components of \plalc KBs. We also studied the properties
of this logic in relation to other probabilistic logics.

We plan to extend this work in several directions.
First, instead of considering the ME-model, 
we could reason over all probability distributions
that satisfy our probabilistic constraints similar to
\cite{Nilsson1986AI,Lukasiewicz99probabilisticdeduction,Hansen2008125}. This will result 
in larger belief intervals in general. A smaller interval
is preferable since it corresponds to a more
precise degree of belief. However, when using
all probability distributions the size of
the interval can be a good indicator for
the variation of the possible beliefs in our query with respect to the knowledge base. 

In some applications it is also useful 
to allow more expressive propositional or relational context languages like those proposed in 
\cite{Kern-IsbernerL04,BonaCF14,BeierleKFP15,Potyka2015}.
Similarly, we can consider other DLs
for our concept language. Indeed, \ALC was chosen as a
prototypical DL for studying the basic properties of our 
framework. Including additional constructors
into the formalism should be relatively simple. In contrast,
considering other reasoning problems beyond subsumption is
less straightforward. Recall, for instance, that if an
\plalc KB \Kmc contains an inconsistent context with positive probability, then \Kmc has no models. It is thus unclear how
to handle the probability of consistency of a KB.

Practical reasoning with \plalc
can be currently performed by
combining existing ME-reasoners\footnote{https://www.fernuni-hagen.de/wbs/research/log4kr/}
with any \alc-reasoner%
\footnote{http://owl.cs.manchester.ac.uk/tools/list-of-reasoners/}
according to
Algorithm~\ref{alg:scepbel}.
Clearly, such an approach can still be further optimized.
We are working on combining the classical
and probabilistic reasoning parts in more 
sophisticated ways.

\bibliographystyle{splncs03}
\bibliography{bib}

\begin{thebibliography}{10}
\providecommand{\url}[1]{\texttt{#1}}
\providecommand{\urlprefix}{URL }

\bibitem{BCM+-07}
Baader, F., Calvanese, D., McGuinness, D.L., Nardi, D., Patel-Schneider, P.F.
  (eds.): The Description Logic Handbook: Theory, Implementation, and
  Applications. Cambridge University Press, 2nd edn. (2007)

\bibitem{BaKP-JWS12}
{Baader}, F., {Knechtel}, M., {Pe{\~n}aloza}, R.: Context-dependent views to
  axioms and consequences of semantic web ontologies. J.\ of Web Semantics
  12--13,  22--40 (2012)

\bibitem{BarnettP08}
Barnett, O., Paris, J.B.: Maximum entropy inference with quantified knowledge.
  Logic Journal of the {IGPL}  16(1),  85--98 (2008)

\bibitem{BeierleKFP15}
Beierle, C., Kern{-}Isberner, G., Finthammer, M., Potyka, N.: Extending and
  completing probabilistic knowledge and beliefs without bias. {KI}  29(3),
  255--262 (2015)

\bibitem{CePe-IJCAR14}
{Ceylan}, I.I., {Pe{\~n}aloza}, R.: The {Bayesian} description logic {BEL}. In:
  Proc.\ of {IJCAR 2014}. LNCS, vol. 8562, pp. 480--494. Springer (2014)

\bibitem{dAFL-08}
d'Amato, C., Fanizzi, N., Lukasiewicz, T.: Tractable reasoning with {Bayesian}
  description logics. In: Proc.\ {SUM 2008}. LNCS, vol. 5291, pp. 146--159.
  Springer (2008)

\bibitem{BonaCF14}
De~Bona, G., Cozman, F.G., Finger, M.: Towards classifying propositional
  probabilistic logics. J.\ of Applied Logic  12(3),  349--368 (2014)

\bibitem{Domingos2009}
Domingos, P.M., Lowd, D.: Markov Logic: An Interface Layer for Artificial
  Intelligence. Synthesis Lectures on Artificial Intelligence and Machine
  Learning, Morgan {\&} Claypool Publishers (2009)

\bibitem{DoMa00}
Donini, F.M., Massacci, F.: {ExpTime} tableaux for $\mathcal{ALC}$. Artificial
  Intelligence  124(1),  87--138 (2000)

\bibitem{Halpern90}
Halpern, J.Y.: An analysis of first-order logics of probability. Artificial
  Intelligence  46,  311--350 (1990)

\bibitem{halpern2004representation}
Halpern, J.Y., Koller, D.: Representation dependence in probabilistic
  inference. JAIR pp. 319--356 (2004)

\bibitem{Hansen2008125}
Hansen, P., Perron, S.: Merging the local and global approaches to
  probabilistic satisfiability. Intern.\ J.\ of Approx.\ Reasoning  47(2),  125
  -- 140 (2008)

\bibitem{Kern-Isberner00d}
Kern-Isberner, G.: Conditionals in nonmonotonic reasoning and belief revision.
  Springer, {LNAI} 2087 (2001)

\bibitem{Kern-IsbernerThimm09a}
Kern-Isberner, G., Thimm, M.: Novel semantical approaches to relational
  probabilistic conditionals. In: Proc.\ {KR 2010}. pp. 382--391. {AAAI} Press
  (2010)

\bibitem{Kern-IsbernerL04}
Kern{-}Isberner, G., Lukasiewicz, T.: Combining probabilistic logic programming
  with the power of maximum entropy. Artif. Intell.  157(1-2),  139--202 (2004)

\bibitem{KlinovP11}
Klinov, P., Parsia, B.: A hybrid method for probabilistic satisfiability. In:
  Proc.\ {CADE} 2011. LNCS, vol. 6803, pp. 354--368. Springer (2011)

\bibitem{Lukasiewicz99probabilisticdeduction}
Lukasiewicz, T.: Probabilistic deduction with conditional constraints over
  basic events. JAIR  10,  380--391 (1999)

\bibitem{LuSt-JWS08}
Lukasiewicz, T., Straccia, U.: Managing uncertainty and vagueness in
  description logics for the semantic web. J.\ of Web Semantics  6(4),
  291--308 (2008)

\bibitem{LutzS10}
Lutz, C., Schr{\"{o}}der, L.: Probabilistic description logics for subjective
  uncertainty. In: Proc.\ {KR} 2010. {AAAI} Press (2010)

\bibitem{Nilsson1986AI}
Nilsson, N.J.: Probabilistic logic. Artificial Intelligence  28,  71--88
  (February 1986)

\bibitem{Nocedal2006}
Nocedal, J., Wright, S.J.: Numerical Optimization. Springer, 2nd edn. (2006)

\bibitem{Paris94}
Paris, J.: The uncertain reasoner's companion -- A mathematical perspective.
  Cambridge University Press (1994)

\bibitem{Potyka2015}
Potyka, N.: Reasoning over linear probabilistic knowledge bases with
  priorities. In: Proc. {SUM} 2015. vol. 9310, pp. 121--136. Springer (2015)

\bibitem{Potyka16}
Potyka, N.: Relationships between semantics for relational probabilistic
  conditional logics. In: Computational Models of Rationality, Essays dedicated
  to Gabriele Kern-Isberner. pp. 332--347. College Publications (2016)

\bibitem{RBLZ-12}
Riguzzi, F., Bellodi, E., Lamma, E., Zese, R.: Epistemic and statistical
  probabilistic ontologies. In: Proc. {URSW-12}. vol. 900, pp. 3--14. CEUR-WS
  (2012)

\bibitem{SS91}
Schmidt{-}Schau{\ss}, M., Smolka, G.: Attributive concept descriptions with
  complements. Artif. Intell.  48(1),  1--26 (1991)

\bibitem{yeung2008information}
Yeung, R.W.: Information theory and network coding. Springer Science \&
  Business Media (2008)

\end{thebibliography}

\clearpage

\section*{Appendix: Proofs}

\consistency*
\begin{proof}
For the ``if'' direction,
let $\logicInt_1,\ldots,\logicInt_n\in\logicInts$ be all the 
\logic-interpretations such that $\meP(\logicInt_i)>0$, $1\le i\le n$.
Then, for every $i,1\le i\le n$ the induced TBox $\Tmc_{\logicInt_i}$ has a classical model 
$\Imc_i=(\Delta^{\Imc_i},\cdot^{\Imc_i})$, by assumption.
It is easy to verify that the \plalc-interpretation 
$\Pmc=(\Imf,P_\Imf)$ defined by  
\[\Imf=\{\Jmc_i=(\Delta^{\Imc_i},\cdot^{\Imc_i},\logicInt_i)\mid 1\le i\le n\}\] 
and $P_\Imf(\Jmc_i)=\meP(\logicInt_i)$ for all
$i,1\le i\le n$ is an ME-model of \Kmc. Thus, \Kmc is consistent.

Conversely, let $\Pmc=(\Imf,P_\Imf)$ be an ME-model of \Kmc. Then, for every 
$\logicInt\in\logicInts$ with $\meP(\logicInt)>0$ there is
a possible world $(\Delta^\Imc,\cdot^\Imc,\altlogicInt^\Imc)$
with $\altlogicInt^\Imc=\logicInt$. Since \Imc is a model
of \Tmc and satisfies all contexts corresponding to the GCIs in $\Tmc_\logicInt$, 
it follows that $(\Delta^\Imc,\cdot^\Imc)\models\Tmc_\logicInt$;
thus, $\Tmc_\logicInt$ must be consistent.
\qed
\end{proof}

\IVT*
\begin{proof}
Assume w.l.o.g. that $\Pmc_i=(\Imf_i,P_i)$, $i=1,2$, are such that
$\Imf_1\cap\Imf_2=\emptyset$: if there exists some $\Imc \in \Imf_1\cap\Imf_2$, 
it suffices to rename the elements in $\Delta^\Imc$ in one of the probabilistic interpretations. 
Given $\lambda \in [0,1]$, define 
$\Pmc_{\lambda}=(\Imf_1 \cup \Imf_2,P_{\lambda})$,
where 
for every $\Imc\in \Imf_1 \cup \Imf_2$,
$$P_\lambda(\Imc)=
	\begin{cases}
		(1-\lambda) P_1(\Imc) & \text{if $\Imc\in\Imf_1$} \\
        \lambda P_2(\Imc) & \text{otherwise.}
	\end{cases}
$$
$\Pmc_{\lambda}$ is consistent with $\Tmc$ since  $\Pmc_1, \Pmc_2$ are. 
We now show that $\Pmc_{\lambda}$ induces the
ME\mbox{-}model of $\Rmc$, which implies that $\Pmc_{\lambda}$ is 
an ME-\plalc-model of \Kmc. 
For all $\logicInt \in \logicInts$, we have
\begin{align*}
 \Pinducedp{\Pmc_{\lambda}}(\logicInt) 
 & {}= \sum_{\Imc \in (\Imf_1 \cup \Imf_2)|_{\logicInt}} P_{\lambda}(\Imc)\\
 & {} = \sum_{\Imc \in \Imf_1|_{\logicInt}} \hspace{-0.2cm} P_{\lambda}(\Imc)
  +\sum_{\Imc \in \Imf_2|_{\logicInt}} \hspace{-0.2cm} P_{\lambda}(\Imc) \\
  & {} = \sum_{\Imc \in \Imf_1|_{\logicInt}} (1 - \lambda) \cdot P_{1}(\Imc)
  + \sum_{\Imc \in \Imf_2|_{\logicInt}}  \lambda \cdot P_{2}(\Imc) \\
  &= (1-\lambda) \cdot  \Pinducedp{\Pmc_{1}}(\logicInt) 
  + \lambda \cdot  \Pinducedp{\Pmc_{2}}(\logicInt) 
  = \meP(\logicInt),
\end{align*}
where the last equation follows from 
$\Pinducedp{\Pmc_{1}} = \Pinducedp{\Pmc_{2}} = \meP$.
Hence, each $\Pmc_{\lambda}$ is indeed a ME-\plalc-model of \Kmc.
For the probability of our subsumption relation, we can derive similarly that
$\Pinducedp{\Pmc_{\lambda}}(\kappa) \propSub_{\Pmc_{\lambda}}(C\sqsubseteq D\mid\kappa)$ is equal to 
\begin{align*}
\sum_{\substack{\Imc\in\Imf_1\cup\Imf_2,\\ \Imc\models\kappa,\Imc\models C\sqsubseteq D}}\hspace*{-6mm} P_\lambda(\Imc) & {} = \sum_{\substack{\Imc \in \Imf_1,\\ \Imc\models\kappa, \Imc\models C\sqsubseteq D}}\hspace*{-6mm} P_{\lambda}(\Imc)
  +\sum_{\substack{\Imc \in \Imf_2,\\ \Imc\models\kappa, \Imc\models C\sqsubseteq D}}\hspace*{-6mm} P_{\lambda}(\Imc) \\
& {} = \Pinducedp{\Pmc_{\lambda}}(\kappa)\left((1-\lambda)p_1 + \lambda p_2 \right).
\end{align*}
For every $p\in[p_1,p_2]$ there exists a $\lambda_p\in[0,1]$ such that 
$p=(1-\lambda_p)p_1+\lambda_p p_2$. Using this value $\lambda_p$ we obtain that
$\propSub_{\Pmc_{\lambda_p}}(C\sqsubseteq D\mid\kappa) = p$
\qed
\end{proof}
In order to prove Theorems~\ref{thm:precise} and~\ref{thm:precise:credulous}
it is useful to consider a restricted class of
\plalc-interpretations in which each context is represented by 
at most one possible world. We call these interpretations \emph{pithy}.

\begin{definition}[pithy]
The \plalc-interpretation $\Pmc{=}(\Imf,P_\Imf)$ is \emph{pithy} if for every $w\in\logicInts$ there is at most one possible world
$(\Delta^\Imc,\cdot^\Imc,\logicInt^\Imc)\in\Imf$
with $\logicInt^\Imc=w$.
\end{definition}
As the following lemma shows, pithy models are sufficient for computing the 
sceptical and credulous degrees of belief of a conditional subsumption relation and,
by extension, the belief interval.

\begin{lemma}
\label{lem:pithy}
Let \Kmc be an \plalc KB, $C,D$ two concepts
and $\kappa \in \logic$ 
such that $\meP(\kappa)>0$.
For every \plalc-model \Pmc of \Kmc there exist pithy \plalc-models $\Qmc_1, \Qmc_2$ of \Kmc such that 
\begin{align*}
\propSub_{\Qmc_1}(C\sqsubseteq D \mid \kappa){\le} \propSub_\Pmc(C\sqsubseteq D \mid \kappa) {\le} \propSub_{\Qmc_2}(C\sqsubseteq D \mid \kappa)
\end{align*}
and $\Pinduced = \Pinducedp{\Qmc_1} = \Pinducedp{\Qmc_2}$.
\end{lemma}
\begin{proof}
Let $\Pmc=(\Imf,P_\Imf)$.
If \Pmc is already pithy, then the result holds trivially.  
Otherwise, there must exist two possible worlds $\Imc,\Jmc\in\Imf$ such that 
$v^\Imc=v^\Jmc$. There are four possible cases:
(i)~$\Imc\models C\sqsubseteq D$ and $\Jmc\models C\sqsubseteq D$;
(ii)~$\Imc\not\models C\sqsubseteq D$ and $\Jmc\not\models C\sqsubseteq D$;
(iii)~$\Imc\models C\sqsubseteq D$ and $\Jmc\not\models C\sqsubseteq D$; and
(iv)~$\Imc\not\models C\sqsubseteq D$ and $\Jmc\models C\sqsubseteq D$.

We construct a new model $\Pmc'$ by removing one of the possible worlds $\Imc,\Jmc$
and redistributing the probability according to these cases, as described next.
For the first three cases, define $\Hmf:=\Imf\setminus\{\Imc\}$ and the
probability distribution
\[
P_\Hmf(\Hmc):=
	\begin{cases}
	  P_\Imf(\Hmc) & \Hmc\not=\Jmc \\
	  P_\Imf(\Imc)+P_\Imf(\Jmc) & \Hmc=\Jmc
	\end{cases}
\]
for all $\Hmc \in \Hmf$. Then $\Pinduced = \Pinducedp{\Pmc'}$ and 
$\Pmc'=(\Hmf,P_\Hmf)$ is still an ME-model of \Kmc.
Since the denominator in \eqref{eq:DefCondProbSubsumpt} is $\meP(\kappa)$
independently of $C\sqsubseteq D$, we have by construction that 
\begin{align*}
\propSub_{\Pmc'}(C\sqsubseteq D \mid \kappa) 
&\le \propSub_\Pmc(C\sqsubseteq D \mid \kappa).
\end{align*}
The case~(iv) is symmetric to (iii), where the possible world \Jmc is removed
instead of \Imc.
Since $\Hmf\subset\Imf$, we can iteratively repeat this process until a pithy model $\Qmc_1$
is obtained.

$\Qmc_2$ can be constructed symmetrically.
\qed
\end{proof}
Notice that since $\Pinduced = \Pinducedp{\Qmc_1} = \Pinducedp{\Qmc_2}$, this 
lemma in particular
implies that for each ME-model there exists a pithy ME-model
that yields a smaller or equal probability for the subsumption relation, and dually
one that yields a larger or equal probability. 
Note also that in each pithy ME-model, 
for all $w \in \logicInts$ with $\meP(w) > 0$, 
there must be exactly one
possible world \Imc with $\logicInt^\Imc=w$ because otherwise 
$P_\Imf$ could not be a probability distribution (the elementary
events could not sum to $1$).
Moreover, 
since a pithy interpretation can contain at most
$|\logicInts|$ possible worlds and the world corresponding to
some $w \in \logicInts$ must have probability $\meP(w)$,
there is only a finite
number of probabilities that pithy models can assign to a given subsumption relation. 
Hence, the infimum and supremum that define the sceptical and  the
credulous degrees of belief actually correspond to minimum and maximum taken by some pithy ME-models.
\begin{corollary}
\label{cor:witness}
Given an \plalc KB \Kmc, two concepts $C,D$ and a context $\kappa$, there exist two
pithy ME-models \Pmc,\Qmc of \Kmc such that 
$\scepbelief_\Kmc(C\sqsubseteq D\mid \kappa)=\propSub_\Pmc(C\sqsubseteq D\mid \kappa)$
and 
$\credbelief_\Kmc(C\sqsubseteq D\mid \kappa)=\propSub_\Qmc(C\sqsubseteq D\mid \kappa)$.
\end{corollary}

\sceptical*
\begin{proof}
For every $w\in\logicInts$, we construct an \plalc-interpretation $\Imc_w$ as follows. 
If $\Tmc_w\models C\sqsubseteq D$,
then $\Imc_w$ is any model $(\Delta^{\Imc_w},\cdot^{\Imc_w},w)$ of $\Tmc_w$; otherwise,
$\Imc_w$ is any model $(\Delta^{\Imc_w},\cdot^{\Imc_w},w)$ of $\Tmc_w$ that does not 
satisfy $C\sqsubseteq D$, which must exist by definition.
Let now $\Pmc_\Kmc=(\Imf,P_\Imf)$ be the
\plalc-interpretation such that $\Imf=\{\Imc_w\mid w\in\logicInts\}$
and $P_\Imf(\Imc_w)=P_{ME}(w)$ for all $w$. Then $\Pmc_\Kmc$ is a model of \Kmc.
Moreover, it holds that 
\begin{align*}
\propSub_{\Pmc_\Kmc}(C\sqsubseteq D\mid \kappa) &{} = \sum_{\Imc_w\models C\sqsubseteq D,w\models\kappa} P_\Imf(\Imc_w) / \meP(\kappa) \\
		&	{} =  \sum_{\Tmc_w\models C\sqsubseteq D}\meP(w) / \meP(\kappa).
\end{align*}
Thus, $\meP(\kappa)\scepbelief_\Kmc(C\sqsubseteq D\mid\kappa) \le
	\sum_{\Tmc_w\models C\sqsubseteq D,w\models\kappa}\meP(w)$. 
If this 
inequality is strict, then w.l.o.g. there must exist a pithy probabilistic model 
$\Pmc=(\Jmf,P_\Jmf)$ of \Kmc such that
$\propSub_\Pmc(C\sqsubseteq D\mid\kappa)< \propSub_{\Pmc_\Kmc}(C\sqsubseteq D\mid\kappa)$ 
(see Lemma~\ref{lem:pithy}). Hence for every $w\in\logicInts$ with
$\meP(w)>0$ there exists exactly one $\Jmc_w\in\Jmf$ with $v^{\Jmc_\Wmc}=w$.
We thus have
\[
\sum_{\Jmc_w\models C\sqsubseteq D}P_\Jmf(\Jmc_w) <
	\sum_{\Imc_w\models C\sqsubseteq D}P_\Imf(\Imc_w).
\]
Since $P_\Imf(\Imc_w)=P_\Jmf(\Jmc_w)$ for all $w$, then there must exist a valuation $v$ such that 
$\Imc_v\models C\sqsubseteq D$ but $\Jmc_v\not\models C\sqsubseteq D$. Since
$\Jmc_v$ is a model of $\Tmc_v$ it follows that $\Tmc_v\not\models C\sqsubseteq D$. 
By construction,
then we have that $\Imc_v\not\models C\sqsubseteq D$, which is a contradiction.
\qed
\end{proof}

\credulous*
\begin{proof}
For every $w\in\logicInts$, construct an \plalc-in\-ter\-pre\-tation $\Imc_w$ as follows. 
If $\Tmc_w\models C\not\,\not\sqsubseteq D$,
then $\Imc_w$ is any model $(\Delta^{\Imc_w},\cdot^{\Imc_w},w)$ of $\Tmc_w$; otherwise,
$\Imc_w$ is any model $(\Delta^{\Imc_w},\cdot^{\Imc_w},w)$ of $\Tmc_w$ that 
satisfies $C\sqsubseteq D$.
Let $\Pmc_\Kmc=(\Imf,P_\Imf)$ be the
\plalc-interpretation with $\Imf=\{\Imc_w\mid w\in\logicInts\}$
and $P_\Imf(\Imc_w)=P_{ME}(w)$ for all $w$. Then $\Pmc_\Kmc$ is a model of \Kmc.
Moreover, it holds that 
\begin{align*}
\propSub_{\Pmc_\Kmc}(C\sqsubseteq D\mid\kappa) &{} = \sum_{\Imc_w\models C\sqsubseteq D, w\models\kappa} P_\Imf(\Imc_w)/\meP(\kappa) \\
			& {} = 1 - \sum_{\Tmc_w\models A\not\sqsubseteq B}\meP(w) / \meP(\kappa).
\end{align*}
That is, $\credbelief_\Kmc(C \sqsubseteq D \mid \kappa)\ge 1-
	\sum_{\Tmc_w\models C\sqsubseteq D,w\models k}\meP(w)/\meP(\kappa)$. 
If this
inequality is strict, then there exists a probabilistic model $\Pmc=(\Jmf,P_\Jmf)$ of \Kmc such that
$\propSub_\Pmc(C\sqsubseteq D \mid\kappa) > \propSub_{\Pmc_\Kmc}(C\sqsubseteq D\mid\kappa)$. 
By Lemma~\ref{lem:pithy}, we can assume w.l.o.g.\ that \Pmc is pithy. Hence for every $w\in\logicInts$ with
$P_{ME}(w)>0$ there is exactly one $\Jmc_w\in\Jmf$ with $v^{\Jmc_\Wmc}=w$, and
thus,
\[
\sum_{\Jmc_w\models C\sqsubseteq D}P_\Jmf(\Jmc_w) >
	\sum_{\Imc_w\models C\sqsubseteq D}P_\Imf(\Imc_w).
\]
Since $P_\Imf(\Imc_w)=P_\Jmf(\Jmc_w)$ for all $w$, there must exist some $v\in\logicInts$ such that 
$\Imc_v\not\models C\sqsubseteq D$ but $\Jmc_v\models C\sqsubseteq D$. As
$\Jmc_v$ is a model of $\Tmc_v$, $\Tmc_v\not\models C\not\,\not\sqsubseteq D$. 
By construction,
we have that $\Imc_v\models C\sqsubseteq D$, which is
a contradiction.
\qed
\end{proof}

\consform*
\begin{proof}
The result is a direct consequence of Definition~\ref{def:cf}
and Theorems~\ref{thm:precise} and~\ref{thm:precise:credulous}. Indeed,
\begin{align*}
\scepbelief_\Kmc(C\sqsubseteq D\mid\kappa) & =
	\sum_{\Tmc_w\models C\sqsubseteq D,w\models\kappa}\meP(w)/\meP(\kappa) \\ & =
    \sum_{w\models\Formula\land\kappa}\hspace*{-1mm}\meP(w)/\meP(\kappa) =
    \meP(\Formula\mid\kappa).
\end{align*}
The case of the credulous degree of belief is analogous.
\qed
\end{proof}

\repinv*
\begin{proof}
Let $\Pmc=(\Imf,P_\Imf)$ be an \plalc-interpretation.
Since $\Models(\Tmc_1) = \Models(\Tmc_2)$,
\Pmc is consistent with $\Tmc_1$ iff
\Pmc is consistent with $\Tmc_2$.
Since  $\Models(\probcons_1) = \Models(\probcons_2)$,
$\probcons_1$ and $\probcons_2$ induce the same ME-model
and $\Pmc_1$ is (ME-)consistent with \probcons iff $\Pmc_2$ is (ME-)consistent with \probcons. Hence, 
\begin{align*}
\scepbelief_{\Kmc_1}(C\sqsubseteq D  \mid \kappa) 
&= \inf_{\Pmc \in \meModels(\Kmc_1)}\propSub_\Pmc(C\sqsubseteq D \mid \kappa) \\
&= \inf_{\Pmc \in \meModels(\Kmc_2)}\propSub_\Pmc(C\sqsubseteq D \mid \kappa) \\
&= \scepbelief_{\Kmc_2}(C\sqsubseteq D  \mid \kappa)
\end{align*}
Analogously, we get that 
$\credbelief_{\Kmc_1}(C\sqsubseteq D  \mid \kappa) 
= \credbelief_{\Kmc_2}(C\sqsubseteq D  \mid \kappa)$
and therefore 
$\belief_{\Kmc_1}(C\sqsubseteq D \mid \kappa) = \belief_{\Kmc_2}(C\sqsubseteq D \mid \kappa)$.
\qed
\end{proof}

\langinv*
\begin{proof}
It suffices to show that for every \plalc-model $\Pmc_1$
of $\Kmc_1$ there exists a \plalc-model $\Pmc_2$
of $\Kmc_2$ such that
$\propSub_{\Pmc_1}(C\sqsubseteq D \mid \kappa)
= \propSub_{\Pmc_2}(C\sqsubseteq D \mid \kappa)$
and \emph{vice versa}.
Given an \plalc-model $\Pmc_1=(\Imf_1,P_{\Imf_1})$
of $\Kmc_1$, we build  $\Pmc_2=(\Imf_2,P_{\Imf_2})$
as follows. For each possible world $\Imc \in \Imf_1$ with probability $p$,
$\Imf_2$ contains a possible world $\Imc'$ with probability $p$ that extends \Imc 
assigning $\textit{false}$ to all new propositional variables
and the empty set to all new role names and concept names. 
Since $C, D, \kappa$ and $\Kmc_2$ depend only on $\sig(\logic^1), \NCp{1}, \NRp{1}$,
$\Pmc_2$ satisfies $\Kmc_2$ and
$\propSub_{\Pmc_1}(C\sqsubseteq D \mid \kappa)
= \propSub_{\Pmc_2}(C\sqsubseteq D \mid \kappa)$
holds. Conversely, consider an \plalc-model $\Pmc_2=(\Imf_2,P_{\Imf_2})$ of $\Kmc_2$. 
We obtain  $\Pmc_1$ from $\Pmc_2$ by restricting the possible 
worlds in $\Imf_2$ to $C, D, \kappa$. As before,
it follows that $\Pmc_1$ satisfies $\Kmc_1$ and
$\propSub_{\Pmc_1}(C\sqsubseteq D \mid \kappa)
= \propSub_{\Pmc_2}(C\sqsubseteq D \mid \kappa)$.
\qed
\end{proof}
In order to prove Independence, we need the following lemma.
It states an independence property of ME-distributions over our
context language.
\begin{lemma}[ME-independence]
\label{lemma:MEindependence}
Let $\probcons_1,\probcons_2$ be two finite sets of probability
constraints such that $\sig(\logic_1)\cap\sig(\logic_2)=\emptyset$, 
and let $\probcons:=\probcons_1\cup\probcons_2$. Then 
$\meP = \mePp{1} \cdot \mePp{2}$.
In particular, for the marginal distributions of $\meP$, we have 
$\meP(\logicInt_i) = \mePi(\logicInt_i)$
for all $\logicInt_i \in \logicIntsi$, $i \in \{1,2\}$.
\end{lemma}
\begin{proof}
Since the signatures of $\logic_i$ are disjoint,
we can denote the valuations of the language over $\sig(\logic_1)\cup\sig(\logic_2)$ by
$(\logicInt_1, \logicInt_2)$, 
$\logicInt_i \in \sig(\logic_i)$.
Let us first consider the marginals of
$P = \mePp{1} \cdot \mePp{2}$.
For all $\logicInt_1 \in \logicIntsp{1}$,
we have 
\begin{align*}
P(\logicInt_1) &= \sum_{\logicInt_2 \in \logicIntsp{2}} P(\logicInt_1, \logicInt_2)
= \mePp{1}(\logicInt_1) \sum_{\logicInt_2 \in \logicIntsp{2}} \mePp{2}(\logicInt_2)\\
&= \mePp{1}(\logicInt_1).
\end{align*}
Symmetrically, we can show that $P(\logicInt_2) = \mePp{2}(\logicInt_2)$.
This means, in particular, that the marginals
of $P$ coincide with the corresponding maximum entropy solutions.
Therefore,
\begin{align*}
H(P) 
&= \sum_{\logicInt_1}  \sum_{\logicInt_2} 
\mePp{1}(\logicInt_1) \mePp{2}(\logicInt_2)
\log (\mePp{1}(\logicInt_1) {\cdot} \mePp{2}(\logicInt_2)) \\
&=\sum_{\logicInt_2} \mePp{2}(\logicInt_2)
\sum_{\logicInt_1}  \mePp{1}(\logicInt_1) 
\log \mePp{1}(\logicInt_1) \\
&+
\sum_{\logicInt_1} \mePp{1}(\logicInt_1) \sum_{\logicInt_2}
 \mePp{2}(\logicInt_2)
\log \mePp{2}(\logicInt_2) \\
&= H(\mePp{1}) + H(\mePp{2}).
\end{align*}
Using the independence bound for entropy
(see, e.g., \cite{yeung2008information}, Theorem 2.39),
we have that
\begin{align*}
H(\meP)
\leq H(\mePp{1})
+ H(\mePp{2}) 
= H(P).
\end{align*}
Hence, it suffices to show that $\mePp{1} \cdot \mePp{2}$ is indeed a model of $\probcons_1 \cup \probcons_2$. 
But this follows immediately from the facts 
that $\mePi$ satisfies $\probcons_i$ and 
that the marginalization of $P$ over 
one logic corresponds to the ME-distribution
over the other.
\qed
\end{proof}
\independence*
\begin{proof}
Let $\Kmc_i=(\probcons_i,\Tmc_i)$ for $i\in\{1,2\}$. Since the signatures of both
KBs are disjoint, we will denote the valuations
over the set of variables $\sig(\probcons_1)\cup\sig(\probcons_2)$ as pairs
$(w_1,w_2)$, where $w_i$ is a valuation over $\sig(\probcons_i)$.
We know from Theorem~\ref{thm:precise} that
\begin{align}
P_{ME}(\kappa)\belief(C\sqsubseteq_{\Kmc_1\cup\Kmc_2} D\mid \kappa) 
& {} = \hspace*{-6mm}
	\sum_{\underset{(\Tmc_1\cup\Tmc_2)_{(w_1,w_2)}\models C\sqsubseteq D}{(w_1,w_2)\models\kappa}}\hspace*{-6mm}\meP((w_1,w_2)) \nonumber \\
    & {} = 	\sum_{\underset{(\Tmc_1)_{w_1}\models C\sqsubseteq D}{(w_1,w_2)\models\kappa}}\hspace*{-4mm}\meP((w_1,w_2)) \label{ind:mono}\\
& {} = \sum_{\underset{(\Tmc_1)_{w_1}\models C\sqsubseteq D}{w_1\models\kappa}}\meP(w_1) \label{ind:ind} \\ 
& {} = \meP(\kappa)\belief(C\sqsubseteq_{\Kmc_1} D\mid \kappa), \nonumber
\end{align}
where 
\eqref{ind:mono} follows from the monotonicity of subsumption
in \ALC TBoxes, and~\eqref{ind:ind} is a consequence of Lemma \ref{lemma:MEindependence}.
\qed
\end{proof}
In order to prove Continuity, we start with a lemma that
states continuity of ME-distributions over \logic.
The proof is analogous to Paris' proof of continuity of maximum entropy reasoning in his
probabilistic logic \cite{Paris94}, which is a sub-logic of our probabilistic logic over \logic.
\begin{lemma}[ME-continuity]
\label{ref:lemma_ME_Continuity}
Let $\probcons$ be a set of probabilistic constraints and
let $(\probcons_i)$ be a sequence of probabilistic constraints
such that $(\Models(\probcons_i))$ converges to $\Models(\probcons)$. Then the sequence $(P_i)$ of ME -Models of $\probcons_i$ converges to the ME\mbox{-}model \meP of $\probcons$.
\end{lemma}
\begin{proof}
For brevity, let 
$M = \Models(\probcons)$ and
$M_i = \Models(\probcons_i)$.
We show that for each $\epsilon > 0$,
there is a $\delta > 0$ such that
$\blaschkeMetric{\Kmc_i, \Kmc} < \delta$
implies that $\varMetric{\mePi, \meP} < \epsilon$.
Consider the set 
$S = \{P \in M \mid 
\varMetric{P, \meP} \geq \frac{\epsilon}{2}\}$
of models of of $\probcons$ that have at least distance $\frac{\epsilon}{2}$ to \meP.
By continuity of the euclidean distance and compactness 
of $M$, $S$ must be compact. Since the entropy function $H$ is continuous, 
the minimum $\nu = \min \{H(\meP) - H(P) \mid P \in S\}$
does exist and $\nu > 0$ by unique maximality of $\meP$.
Since $H$ is defined on a compact set (the set of probability distributions over \logic), $H$ is uniformly continuous.
Therefore there exists a $\delta > 0$ such that for all distributions 
$P_1, P_2$ over $\logic$, $\varMetric{P_1, P_2} < \delta$ implies that $H(P_1) - H(P_2) < \min \{ \frac{\epsilon}{2}, \frac{\nu}{2}\}$. In partiular, we can
assume that $\delta < \frac{\epsilon}{2}$.
Now, if $\blaschkeMetric{M_i,M} < \delta$, 
there is a $P_i \in M_i$
such that $\varMetric{P_i,\meP} < \delta$ 
and a $P \in M$
such that $\varMetric{\mePi,P} < \delta$. Hence,
\begin{align*}
&H(\meP) < 
H(P_i) + \frac{\nu}{2} \leq H(\mePi) + \frac{\nu}{2},\\
&H(\mePi) < 
H(P) + \frac{\nu}{2} \leq H(\meP) + \frac{\nu}{2}
\end{align*}
and therefore $|H(\mePi) - H(\meP)| < \frac{\nu}{2}$.
In particular,
\begin{align*}
&|H(P) - H(\meP)| \\
&\leq |H(P) - H(\mePi)|+ |H(\mePi) - H(\meP)| \\
&< \frac{\nu}{2} + \frac{\nu}{2} = \nu.
\end{align*}
By definition of $\nu$, we can conclude that $P \in M \setminus S$ and therefore 
$\varMetric{\meP, P} < \frac{\epsilon}{2}$.
Hence, 
\begin{align*}
\varMetric{\meP, \mePi} 
\leq \varMetric{\meP, P} + \varMetric{P, \mePi}
< \frac{\epsilon}{2} + \delta  < \epsilon. \quad 
\end{align*}
\qed
\end{proof}

\continuity*
\begin{proof}
Let $\Kmc=(\probcons,\Tmc)$ and $\Kmc_i=(\probcons_i,\Tmc_i)$.
By assumption, $(\Models(\probcons_i))$ converges to $\Models(\probcons)$.
Hence, Lemma \ref{ref:lemma_ME_Continuity} implies that 
the probability distributions induced by ME-models of $\Kmc_i$ converge to \meP, which, in turn, is the probability distribution induced by all ME-models of $\Kmc$.
Hence, infimum and supremum of the conditonal probability 
of $C\sqsubseteq D$ given $\kappa$ with respect to $\Kmc_i$ 
will converge to infimum and supremum with respect to $\Kmc$.
\qed
\end{proof}

\end{document}